\pdfoutput=1

\documentclass[11pt]{article}

\usepackage[]{ACL2023}

\usepackage{times}
\usepackage{latexsym}

\usepackage[T1]{fontenc}

\usepackage[utf8]{inputenc}

\usepackage{microtype}

\usepackage{inconsolata}

\usepackage{amsmath}
\usepackage[utf8]{inputenc} 
\usepackage[T1]{fontenc}    
\usepackage{hyperref}       
\usepackage{cleveref}
\usepackage{url}            
\usepackage{booktabs}       
\usepackage{amsfonts}       
\usepackage{nicefrac}       
\usepackage{microtype}      
\usepackage{xcolor}         
\usepackage{listings}
\usepackage{algorithm}
\usepackage{algorithmic}
\usepackage{graphicx}
\usepackage{multicol}
\usepackage{CJKutf8}
\usepackage{booktabs}
\usepackage{subcaption}
\usepackage[utf8]{inputenc}
\usepackage{titlesec}
\usepackage[normalem]{ulem}
\usepackage{xspace}
\usepackage{mathtools}
\usepackage{multirow}
\usepackage{graphicx}
\usepackage{subcaption}
\usepackage{pythonhighlight}
\usepackage[font=small]{caption}
\usepackage{amsthm}
\usepackage{tabularray}
\usepackage{regexpatch}

\newcommand{\adv}{adv-ICL\xspace}

\newtheorem{prop}{Proposition}

\title{Prompt Optimization via Adversarial In-Context Learning
}

\author{%
  Xuan Long Do$^{1,3}$\footnotemark[1]\thanks{\;\;Equal contribution.}\;,\; Yiran Zhao$^{1}$\footnotemark[1]\;,\; Hannah Brown$^{1}$\footnotemark[1]\;,\; Yuxi Xie$^{1}$,\; James Xu Zhao$^{1}$, \\ \textbf{Nancy F. Chen$^{3}$,\; Kenji Kawaguchi$^{1}$,\; Michael Shieh$^{1}$\footnotemark[2]\thanks{\;\;Equal advising.}\;,\; Junxian He$^{2}$\footnotemark[2]} \\
  $^{1}$National University of Singapore, \\ 
  $^{2}$Hong Kong University of Science and Technology, \\ 
  $^{3}$Institute for Infocomm Research (I$^2$R), A*STAR \\
  \texttt{\{xuanlong.do, zhaoyiran, hsbrown, xieyuxi, xu.zhao\}@u.nus.edu,} \\
  \texttt{\{kenji, michaelshieh\}@nus.edu.sg,} \\
  \texttt{junxianh@cse.ust.hk, nfychen@i2r.a-star.edu.sg} \\
}

\begin{document}
\maketitle
\begin{abstract}

We propose a new method, Adversarial In-Context Learning (\adv{}\footnote{Our codes will available at \small{\url{https://github.com/zhaoyiran924/Adv-In-Context-Learning}.}}), to optimize prompts for in-context learning (ICL). Inspired by adversarial learning, \adv{} is implemented as a two-player game between a generator and discriminator, with LLMs acting as both. In each round, given an input prefixed by task instructions and several exemplars, the generator produces an output. The discriminator then classifies the generator's input-output pair as model-generated or real data. Based on the discriminator's loss, a prompt modifier LLM proposes possible edits to the generator and discriminator prompts, and the edits that most improve the adversarial loss are selected. We show that applying \adv{} results in significant improvements over state-of-the-art prompt optimization techniques for both open and closed-source models on $13$ generation and classification tasks including summarization, arithmetic reasoning, machine translation, data-to-text generation, and the MMLU and big-bench hard benchmarks. In addition, our method is computationally efficient, easily extensible to other LLMs and tasks, and effective in low-resource settings

\end{abstract}

\section{Introduction}

Generative Adversarial Networks (GANs) and adversarial learning~\citep{goodfellow2014generative} have driven significant progress across a range of domains, including image generation~\citep{goodfellow2014generative, radford2015unsupervised, arjovsky2017wasserstein}, domain adaptation~\citep{ganin2016domain, tzeng2017adversarial, xie2017controllable, louppe2017learning}, and enhancing model robustness~\citep{szegedy2013intriguing,biggio2013evasion,carlini2017towards,madry2018towards}. At its core, adversarial learning frames training as a minimax game between a \emph{generator} and a \emph{discriminator}. The generator aims to generate output realistic enough that the discriminator classifies it as real (i.e., not generated), while the discriminator aims to accurately differentiate between generator output and real training samples. After each round, the parameters of both models are updated based on an adversarial loss, and the process repeats. As the generator improves, the discriminator improves alongside it, finding ``weak spots" in generator output that may go undiscovered in non-adversarial training, ultimately resulting in better generator outputs.

Despite success in other domains, applying adversarial learning to pre-training LLMs is impractical due to the data and computational overheads associated with training two models. Particularly for novel tasks where data is often scarce, it is desirable to have methods that can improve model performance using limited data. In this work, we solve this problem by applying adversarial learning to \emph{in-context learning (ICL)}~\citep{brown2020language, chowdhery2022palm, touvron2023llama, beltagy-etal-2022-zero,liu2023pre}, which has shown to be an effective method to improve model performance with few training samples. Though, effective, ICL has shown to be sensitive to changes in prompts~\cite{deng2022rlprompt,pryzant2023automatic}. We introduce \emph{Adversarial In-Context Learning} (\adv{}), which applies insights from adversarial learning to prompt optimization for ICL. \adv{} keeps model parameters fixed and instead updates model prompts in an adversarial manner. This alleviates compute and data requirements, while still allowing improvements in model performance.

\begin{figure*}[ht]
\centering
\includegraphics[width=0.95\textwidth]{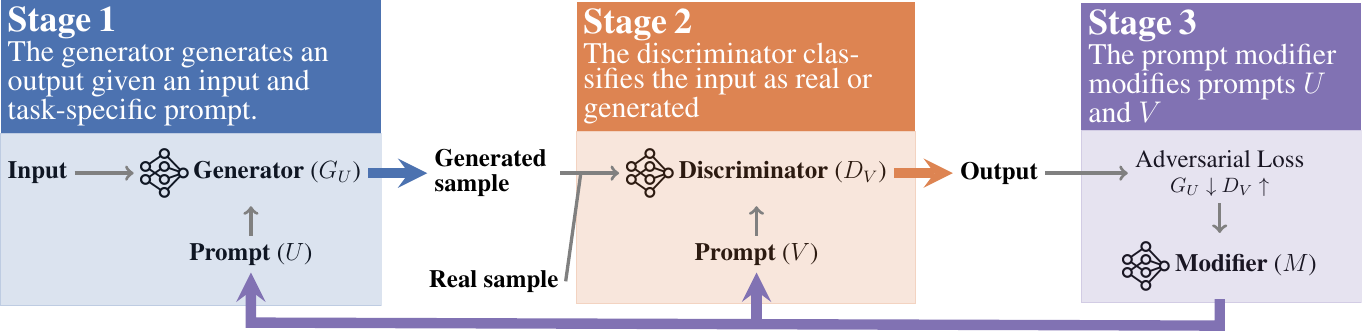}
  \vspace{-2mm}
  \caption{\adv{} orchestrates a minimax game between a \textit{Generator} and a \textit{Discriminator}, both powered by LLMs with few-shot prompts. The Generator crafts responses to unlabeled examples, while the Discriminator distinguishes between generated and ground truth outputs. Updates are made by a \textit{Prompt Modifier} which modifies prompts based on the adversarial loss.}
  
\label{fig:model}
\end{figure*}

\adv{} uses an adversarial objective and three main modules, implemented as LLMs, to optimize a model's prompt for a given task, as shown in Figure \ref{fig:model}.  The first module is a generator ($G$), which is tasked with generating realistic, task appropriate output given a task instruction and an input. The second is a discriminator ($D$) which has the goal of classifying its inputs as real or produced by $G$. Finally, there is a prompt modifier $M$ which is responsible for updating the prompts to $G$ and $D$. As in typical adversarial learning, the learning objective is set up as a minimax game between $G$ and $D$. In each round, $G$ produces an output based on an input and a prompt consisting of a task instruction and several example inputs and outputs. $D$ then classifies the pair constructed of the original input and $G$'s output as generated or real. Finally, $M$ produces a number of possible updates to $G$ and $D$'s prompts, the updates that most improve the adversarial loss from $D$'s classification are selected, and the procedure repeats. Through this iterative update procedure \adv{} is able to improve $G$'s prompt, improving task performance.

We evaluate \adv{} on $13$ tasks with various open and closed-source LLMs, finding that \adv{} outperforms other prompt optimization techniques by large margins across model configurations and tasks. For instance, we increase the accuracy of ChatGPT~\citep{openai2022chatgpt} from 71.0\% to 74.0\% on MMLU~\citep{hendrycks2020measuring}, 79.9\% to 82.3\% on GSM8K~\citep{cobbe2021training}, and 72.1\% to 74.0\%  on BBH~\citep{suzgun2022challenging}.  Importantly, \adv{} requires very few iterations and training samples, increasing performance significantly after only five rounds of training on twenty data points. Finally, \adv{} is easy to implement, encouraging its use in real-world applications.

\section{Adversarial In-Context Learning}

\begin{figure}[htb]
    \centering
    \includegraphics[width=0.62\linewidth]{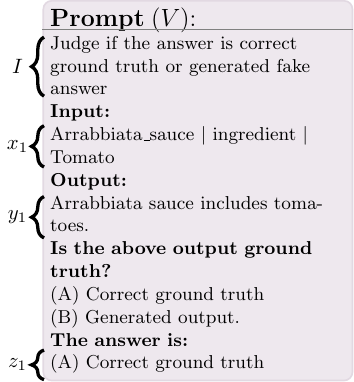}
    \caption{An example of a task prompt for the discriminator $D_V$ with prompt components labeled.}
    \label{fig:d_prompt}
\end{figure}


\subsection{Background: In-Context Learning}
Large Language Models (LLMs) \citep{brown2020language, chowdhery2022palm, touvron2023llama, openai2023gpt4} have demonstrated strong downstream task performance through conditioning on a small number of demonstrations in the input prompt, a paradigm referred as \emph{in-context learning (ICL)} \citep{beltagy-etal-2022-zero,liu2023pre}. ICL streamlines the adaptation of a general-purpose LLM to a specific task without the need for feature engineering or additional model training.
Formally, given a specific task, consider an LLM $G_U$ (the generator), driven by a prompt $U=(I^G, x^G_1, y^G_1, \cdots, x^G_k, y^G_k)$, where $I^G$ is the task instruction, $x^G_i$ is a sample input, and $y^G_i$ is the corresponding sample output. $G_U$'s output for a new input $x$, then, is determined by the instruction and the exemplars in $U$, making the choice of $U$ crucial in determining $G_U$'s downstream performance~\citep{deng2022rlprompt,pryzant2023automatic}.



\subsection{Adversarial Training Objective}
\adv{} optimizes the generator's prompt using an adversarial approach inspired by GANs~\citep{goodfellow2014generative}---in particular cGAN~\citep{mirza2014conditional} and BiGAN~\citep{donahue2016adversarial} where the discriminator deals with the conditional  and joint distribution of an input and output. As for GANs, it is essential to optimize both the discriminator and generator in the \adv{}~framework concurrently, to make sure  they reach a desired optimal state. To assess the output of $G_U$, we employ a discriminator, $D_V$, which attempts to classify $G_U$'s output as real or generated.

Like $G_U$, $D_V$ is an LLM driven by a prompt $V=(I^D, x^D_1, y^D_1, z^D_1, \cdots, x^D_k, y^D_k, z^D_k)$, where $I^D$ is a task instruction, $x^D_i$ a sample input, $y^D_i$ its corresponding output, and $z^D_i$ a label of ``real'' or ``generated'' representing whether $y^D_i$ is from a real sample or generated by $G_U$. $D_V$ uses a GAN-inspired loss function $\mathcal{J}$, formally defined as:

\begin{equation}
\begin{multlined}
\mathcal{J}(D_V, G_U)=\mathbb{E}_{x, y\sim p_{data}} \log \Big( D_V(x, y) \Big) \\+ \mathbb{E}_{x\sim p_{data}} \log\Big( 1 - D_V\big(x, G_U(x)\big) \Big) 
\label{eq:loss}
\end{multlined}
\end{equation}

where $p_{data}$ is the distribution of real data. Note that $D_V$ is designed for the binary decision problem of classifying the input as generated or real. As shown in \Cref{fig:d_prompt}, in our prompt, we represent the choices as two options: ``(A) real'' or ``(B) generated''. This allows us to evaluate the classification probability using the generation probability of option (A), where $D_V(x, y)=1$ indicates a real sample. Therefore, in order for $G_U$ to improve its performance, its goal is for $D_V$ to mis-classify its outputs as real as often as possible (i.e. minimizing $\mathcal{J}$). In contrast, $D_V$'s objective is to increase $\mathcal{J}$, indicating improved classification ability. Formally, this adversarial training objective can be expressed as the following minimax game: 

\begin{equation}
    \min_{U} \max_{V} \mathcal{J}(D_V, G_U)
\end{equation}

Since the discriminator is powered by a large language model with enough capacity, achieving the optimal solution for this minimax objective indicates that the generator's output, when paired with its input, is indistinguishable from a real sample.

\begin{algorithm*}[ht]
\renewcommand{\algorithmicrequire}{\textbf{Input:}}
\renewcommand{\algorithmicensure}{\textbf{Output:}}
    \caption{Adversarial In-Context Learning Optimization}
    \begin{algorithmic}[1]
    \REQUIRE $U=(I^G, x^G_1, y^G_1, \cdots, x^G_k, y^G_k)$, $V=(I^D, x^D_1, y^D_1, z^D_1, \cdots, x^D_k, y^D_k, z^D_k)$.
    \REQUIRE Generator $G_U$, Discriminator $D_V$, Prompt Modifier $M$. 
    \REQUIRE \textcolor{black}{Prompts for $M$ to sample new instructions or demonstrations $P_i/P_d$}.
    \REQUIRE Training iterations $T$, samples used per iteration $m$, number of new sampled prompts $r$.
    \REQUIRE Set of limited samples $S$
        \FOR {$T$ training iterations}
            \STATE Sample $m$ data points from $S$ to compute $J(G_U, D_V,m)$.
            \STATE \texttt{// Optimize the instruction $I^D$ for $D_V$} \\
            \STATE \textcolor{black}{Generate $r$ new instructions: \{$I_1$, $I_2$,...,$I_r$\} $= M(I^D, P_i)$}. 
            \STATE Substitute $I_n$ to $V$ $\forall n \in \{1,2,...,r\}$ to compute the loss $J_n(G_U, D_V,m)$
            \STATE $J_j = \max_n J_n(G_U, D_V, m)$
            \STATE Update $I^D$ by $I_j$ if $J_j > J$. 
            \STATE \texttt{// Optimize the demonstrations ($x^D_{i}, y^D_{i}, z^D_{i}$) $\forall i$ for $D_V$} \\
            \FOR {$i \in range(k)$}
               \STATE \textcolor{black}{Generate $r$ new demonstrations: \{($x_{i1}, y_{i1}, z_{i1}$),...,($x_{ir}, y_{ir}, z_{ir}$)\} $= M((x^D_{i}, y^D_{i}, z^D_{i}), P_d)$.} \\
               \STATE Substitute ($x_{in}, y_{in}, z_{in}$) to $V$ $\forall n \in \{1,2,...,r\}$ to compute the loss $J_{in}(G_U, D_V,m)$
               \STATE $J_{jn}=\max_i J_{in}$
               \STATE Update ($x^D_{i}, y^D_{i}, z^D_{i}$) by ($x_{ij}, y_{ij}, z_{ij}$) if $J_{jn} > J$.
            \ENDFOR
            \STATE \texttt{// Similarly optimize $U$ for $G_U$ so that $J(G_U, D_V,m)$ decreases.}
            \STATE ...
        \ENDFOR
    \ENSURE The optimized prompt $U$ for the Generator $G_U$.
    \end{algorithmic}
    \label{algo:adversarial}
\end{algorithm*}

\subsection{Adversarial In-Context Learning Optimization}\label{sec:adv-optimization}

Whereas GANs optimize model parameters with backpropagation, \adv{} does not update $G_U$ and $D_V$'s parameters, instead updating their prompts in each training iteration. This requires a number of differences in our optimization process. First, we consider a setting where we have access only to model outputs and generation probabilities, making it impossible to use backpropagation to update $U$ and $V$. Therefore, we employ a third LLM to serve as the \emph{prompt modifier}, $M$. Given a prompt's task instruction $I$ or demonstration $(x,y)$ as input, $M$ generates $r$ possible variations. The adversarial loss is recomputed for each variation by substituting the variation into the original prompt, and the modification that improves the adversarial loss the most is returned, following \citet{gonen2022demystifying}. 

We refer to our optimization algorithm as \emph{Adversarial In-Context Learning Optimization} (adv-ICL; \Cref{algo:adversarial}), which can be seen in pseudocode form in \Cref{algo:adversarial}. The entire process is as follows: Given the initial generator prompt $U$, and discriminator prompt $V$, we run $T$ training iterations. At each iteration, we first sample $m$ pairs of data points from our training samples to compute the adversarial training loss $\mathcal{J}(G_U, D_V,m)$. We then optimize the loss by using $M$ to modify both the task instruction and demonstration portions of the prompts for the discriminator and generator. 


\subsection{Theoretical Analysis} \label{sec:theoretical-results}
In this section, we present an analysis of whether a minimax objective can achieve equilibrium in in-context learning as is possible in the original GAN scenario. Let $p_{data}$ be the distribution of the training data, and $p_g$ be of the generated data from $G$. We assume $D$, $G$, and $M$ are models with infinite capacity and strong enough in-context learning capabilities, where the prompts powering $D$ and $G$ are iteratively updated using $M$ following \cref{algo:adversarial}. We further assume that: \emph{(i) $M$ is powerful enough to modify the initial prompt of $D/G$, covering all possible prompt variants; (ii) There exists a prompt $\mathcal{P}$ for $D/G$ that given $\mathcal{P}$, $D/G$ can achieve the globally optimal result; (iii) $M$ can generate $\mathcal{P}$ by which $D/G$  achieves the globally optimal result}. We prove the following:

\begin{prop}{(Motivated by \cite{goodfellow2014generative})} \label{prop:prop3}
If $G$ and $D$ have enough capacity, and at each training step, the discriminator is allowed to reach its optimum $D^*$ given $G$, and $p_g$ is updated so as to improve the criterion 

\begin{equation}
\begin{multlined}
\mathcal{J}(D^*, G)=\mathbb{E}_{x, y\sim p_{data}} \log \Big( D^*(x, y) \Big) \\ + \mathbb{E}_{x\sim p_{data}} \log\Big( 1 - D^*\big(x, G(x)\big) \Big) 
\end{multlined}
\end{equation}

then $p_g$ converges to $p_{data}$.
\end{prop}

The full proof for proposition~\ref{prop:prop3} can be found in \Cref{sec:theoretical-proof}. We conclude that with strong enough in-context learning LLMs $D,G,M$, \adv{} converges. In practice, convergence in adversarial training must be studied empirically. 


\subsection{Zero-shot Prompt Modification}

\begin{figure}[htb]
\centering
\includegraphics[width=\linewidth]{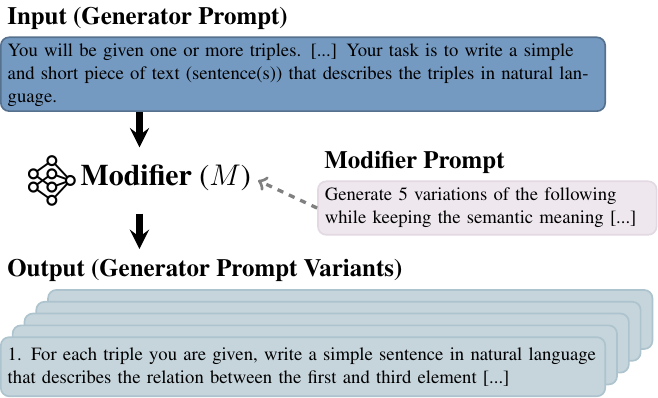}
  \caption{Example of how the prompt modifier generates new versions of $G_U$'s prompt $U$ including new task instructions and new data examples. Full prompts used for $M$ are in \Cref{appendix:prompts-for-prompt-modifier}. 
  }
\label{fig: modifier}
\end{figure}

We leverage LLM instruction-following abilities to generate $r$ variants of a task instruction/demonstration in a zero-shot manner. We use different prompt templates for generating instructions, open-ended question-answer pairs, and multiple-choice question pairs. \textcolor{black}{Specifically, suppose that $P$ is an input task instruction or demonstration (open-ended or multiple choice) to be modified and $M$ is the prompt modifier generating the variants, with its modifier prompt $P_M$. Then, $M$ is prompted to generate $r$ variants of $P$: $\{P_1,...,P_r\} = M(P, P_M)$.} An example is shown in \Cref{fig: modifier}, with full prompts in \Cref{appendix:prompts-for-prompt-modifier}. We also tested using successful prompts from previous optimizations as feedback for the next iteration (\Cref{appendix:extra-analysis}), following \citet{pryzant2023automatic,yang2023large}, but found that performance fell short compared to not integrating them.


\section{Experimentation}
\subsection{Experimental Setup} \label{subsec:experimental-setup}
\paragraph{Datasets.} We test \adv{} on $13$ traditional NLP tasks in four main categories: \emph{generation}, \emph{classification}, \emph{reasoning}, and challenging NLP \emph{evaluation suites}. For generation, we select XSUM~\citep{narayan2018don} and CNN/Daily Mail (CNN for short)~\citep{nallapati2016abstractive} as our \emph{text summarization} benchmarks; WebNLG~\citep{gardent2017creating} and E2E NLG~\citep{novikova2017e2e} as our \emph{data-to-text generation} datasets; and LIRO (RO → EN)~\citep{dumitrescu2021liro} and TED Talks (IT→ JA)~\citep{Ye2018WordEmbeddings} as our \emph{machine translation} benchmarks. In the classification category, we use YELP-5~\citep{yelp}, COPA~\citep{roemmele2011choice} and WSC~\citep{levesque2012winograd}. For reasoning tasks,  GSM8K~\citep{cobbe2021training} and SVAMP~\citep{patel2021nlp} are chosen as arithmetic reasoning benchmarks. Finally, we also evaluate our method on two challenging \emph{evaluation suites}: MMLU~\citep{hendrycks2020measuring} and BIG-bench Hard (BBH)~\citep{suzgun2022challenging}. 
Due to computational and budget limitations, except for GSM8K and SVAMP, each benchmark is evaluated on a maximum of 1,000  test samples randomly chosen from the test set. 
In our preliminary experiments, we found that the empirical results on the sampled test set is aligned with performance on the whole test set. 
The exact number of testing samples for each task is presented in \Cref{appendix:benchmark_statistics}.

A main advantage of ICL is that it can generalize to new tasks with limited training examples, as may be the case for novel tasks. To make our method applicable in such settings, we use 20 labeled samples during training. For our baseline methods, we assume access to at most 100 labeled data samples for each benchmark except BBH, similar to previous prompt optimization works~\citep{xu-etal-2022-gps, pryzant2023automatic}. For BBH, we assume access to three chain-of-thought data samples per task.

\paragraph{Backbone Models.} We test widely-used open and closed-source LLMs as our backbone models. For open-source, we use \emph{Vicuna-13B v1.5} \citep{zheng2023judging} ---a chat model fine-tuned on top of \emph{LLaMa-2} \citep{touvron2023llama2}. For closed-source, we use \emph{text-davinci-002} and \emph{ChatGPT (gpt-3.5-turbo-0613)}~\citep{openai2022chatgpt} built on top of \emph{GPT-3} \citep{brown2020language}. For each backbone model except ChatGPT, we use the same model for the generator, discriminator, and prompt modifier in the \adv{}~setup. Since ChatGPT does not provide the probabilities of its generated tokens, which is required for computing the adversarial loss, we employ \emph{text-davinci-002} as the discriminator, and ChatGPT is the generator and the prompt modifier.

\paragraph{Baselines.} We compare \adv{} with six baselines: (i) \emph{Few-shot} prompting, with Chain-of-Thought (CoT) \citep{wei2022chain} for reasoning tasks; (ii) 
Utilizing ROUGE-L score \citep{lin2004rouge} (\emph{ROUGE-L}) as the criteria to optimize the instruction and demonstrations  for each task on a small sampled labeled set; (iii) Similarly, using Perplexity (\emph{Perplexity}) as the criteria following \citet{gonen2022demystifying}; (iv) Genetic Prompt Search (\emph{GPS}) \citep{xu-etal-2022-gps}, a genetic optimization method based on the log-logits or accuracy; (v) Automatic Prompt Optimization (\emph{APO}) \citep{pryzant2023automatic}, which uses data to generate text “gradients” evaluating the current prompt, and then utilize them to signal the models to edit the prompt in the opposite semantic direction. {(vi) Automatic Prompt Engineer (APE) \citep{zhou2022large}, which automatically generates instructions and selects via evaluation scores.}\footnote{{As APE only polishes task instruction, we compare APE with Adv-ICL on GSM8K, MMLU and WebNLG.}}

We ensure that all methods use a similar number of labeled samples, while the exact number of training samples depends on the specific algorithm. For GPS and APO, we sample $32$ and $50$ labeled data examples for validation, following \citep{xu-etal-2022-gps, pryzant2023automatic}. For ROUGE-L and Perplexity, we sample $80$ data examples for validation. For YELP, WSC, GSM8K, SVAMP, where the benchmarks do not have enough labeled examples, we sample from their limited training set instead. APO requires additional training data for error samples. For fair comparisons, we use the same training data with {\adv{}}. More implementation details for baselines are presented in \Cref{appendix:baseline-implementation}.

\paragraph{Prompt Initialization.} 
We follow prior works to employ a set of initialized prompts. 
For MMLU and BBH, we employ the open-sourced prompts that come with the original papers. For GSM8K and SVAMP, we follow the chain-of-thought paper \citet{wei2022chain} which employs human-written prompts. For the remaining benchmarks, we utilize prompts from Super-NaturalInstructions~\citep{wang2022super}, in which instructions and demonstrations are chosen by domain experts. All the initial prompts are also used for our baseline \emph{few-shot} experiments. The exact number of shots used for each benchmark is presented in \Cref{appendix:number-of-shots}.

\paragraph{Evaluation Metrics.} For the generation tasks, we evaluate the performance by ROUGE-L score \citep{lin2004rouge}, following~\citet{wang2022super}. For classification tasks, we use accuracy as the evaluation metric. For MMLU and BBH, we follow \citet{hendrycks2020measuring, suzgun2022challenging} and report the averaged performance among tasks.

\paragraph{Hyperparameters.} Following the hyperparameter selection results in \Cref{sec:further-study}, we set number of training iterations $T=3$ and training samples per iteration $m=5$ for all tasks except BBH, where we set $T=3, m=3$ given that the training set contains only $3$ samples. In all experiments, the prompt modifier samples from $r=5$ prompts.



\subsection{Main Results}
We present the main empirical results on a set of classification, generation and reasoning tasks in \Cref{table:result_gen}, MMLU in \Cref{tab:chat_mmlu}, and BBH in \Cref{fig:chat_bbh}.

\begin{table*}
  \centering
\footnotesize
  \scalebox{0.72}{
  \begin{tabular}{cc|cccccc|ccc|cc}
    \toprule
  \multirow{2}*{\textbf{{Models}}} & \multirow{2}*{\textbf{{Method}}} & \multicolumn{2}{c}{\textbf{\normalsize{Summarization}}}  & \multicolumn{2}{c}{\textbf{\normalsize{Data-to-Text}}} & \multicolumn{2}{c}{\textbf{\normalsize{Translation}}} & \multicolumn{3}{c}{\textbf{\normalsize{Classification}}} & \multicolumn{2}{c}{\textbf{\normalsize{Reasoning}}} \\\specialrule{0em}{2pt}{2pt}
    &  & \textbf{XSUM} & \textbf{CNN} & \textbf{WebNLG} & \textbf{E2E NLG}  & {\textbf{LIRO}} & {\textbf{TED Talks}} & \textbf{YELP Review}  & \textbf{COPA}  & \textbf{WSC} & \textbf{GSM8K} & \textbf{SVAMP}\\
    \midrule
   \multirow{6}*{\begin{tabular}[c]{@{}l@{}} {\rotatebox{90}{text-davinci-002}} 
\end{tabular}}
   & Few-shot & 25.5 & 20.8 & 60.8 & 47.1 & 78.3 & 37.7 & 71.1 & 87.9 & 67.7 & 47.3  & 70.0 \\ 
   & ROUGE-L & 25.8 & 21.1 & 61.1 & 47.5 & 77.6 & 38.2 & 70.6 & 87.8 & 66.9 & 47.1 & 69.8 \\
   & Perplexity & 26.2 & 21.4 & 62.2 & 49.3 & 78.5 & 39.0 & 70.9 & 88.6 & 67.3 & 47.5 & 70.4 \\
   & GPS & 27.1 & 21.5 & 61.9 & 49.1 & 78.8 & 39.4 & 71.3 & 87.4 & 67.1 & 48.1 & 70.5 \\
   & APO & 26.8 & 22.1 & 62.3 & 49.2 & 78.9 & 40.2 & 71.1 & 88.8 & 68.3 & 46.9 & 69.3 \\
     & {\adv{}}  &  \textbf{30.9}\textcolor[RGB]{84,123,71}{$\uparrow$$3.8$} &  \textbf{23.4}\textcolor[RGB]{84,123,71}{$\uparrow$$1.3$}  & \textbf{65.4}\textcolor[RGB]{84,123,71}{$\uparrow$$3.1$} & \textbf{50.8}\textcolor[RGB]{84,123,71}{$\uparrow$$1.5$} & \textbf{81.2}\textcolor[RGB]{84,123,71}{$\uparrow$$2.3$} & \textbf{42.1}\textcolor[RGB]{84,123,71}{$\uparrow$$1.9$}  & \textbf{74.4} \textcolor[RGB]{84,123,71}{$\uparrow$$3.1$} &  \textbf{92.2} \textcolor[RGB]{84,123,71}{$\uparrow$$3.4$} & \textbf{73.8}\textcolor[RGB]{84,123,71}{$\uparrow$$5.5$} & \textbf{50.8} \textcolor[RGB]{84,123,71}{$\uparrow$$2.7$} & \textbf{72.5} \textcolor[RGB]{84,123,71}{$\uparrow$$2.0$} \\\midrule

 \multirow{6}*{\begin{tabular}[c]{@{}l@{}}{\rotatebox{90}{Vicuna v1.5}}
\end{tabular}} 
   & Few-shot & 18.9 & 16.4 & 52.5 & 35.3 & 72.1 & 32.6 & 71.0 & 77.8 & 54.4 & 40.7 & 45.1 \\
   & ROUGE-L & 18.9 & 16.6 & 52.7 & 35.2 & 72.6 & 32.9 & 70.9 & 76.7 & 54.1 & 40.4 & 44.8 \\
   & Perplexity & 19.1 & 16.9 & 52.8 & 35.0 & 72.7 & 33.0 & 71.0 & 77.9 & 54.7 & 41.4 & 46.2 \\
   & GPS & 19.7 & 16.9 & 53.0 & 35.9 & 73.2 & 33.0 & 71.3 & 78.2 & 55.0 & 41.7 & 45.7 \\
   & APO & 19.5 & 17.1 & 53.7 & 36.3 & 73.1 & 32.9 & 70.2 & 78.3 & 54.4 & 41.4 & 46.3 \\
   & {\adv{}}  & \textbf{21.1}\textcolor[RGB]{84,123,71}{$\uparrow$$1.4$} & \textbf{19.3}\textcolor[RGB]{84,123,71}{$\uparrow$$2.2$} & \textbf{59.3}\textcolor[RGB]{84,123,71}{$\uparrow$$5.6$} & \textbf{41.9}\textcolor[RGB]{84,123,71}{$\uparrow$$5.6$} & \textbf{73.4}\textcolor[RGB]{84,123,71}{$\uparrow$$0.2$} & \textbf{35.2}\textcolor[RGB]{84,123,71}{$\uparrow$$2.2$} & \textbf{73.6}\textcolor[RGB]{84,123,71}{$\uparrow$$2.3$} & \textbf{81.6}\textcolor[RGB]{84,123,71}{$\uparrow$$3.3$} & \textbf{58.2}\textcolor[RGB]{84,123,71}{$\uparrow$$3.2$}  & \textbf{43.9}\textcolor[RGB]{84,123,71}{$\uparrow$$3.2$} & \textbf{48.4}\textcolor[RGB]{84,123,71}{$\uparrow$$3.3$} \\
    \midrule
   \multirow{6}*{\begin{tabular}[c]{@{}l@{}} {\rotatebox{90}{ChatGPT}} \\
\end{tabular}} 
   & Few-shot & 25.2 & 21.3 & 60.9 & 48.3 & 78.8 & 41.7  & 69.8 & 94.4 & 69.8 & 79.4  &79.3 \\ 
   & ROUGE-L & 25.1 & 21.2 & 60.7 & 48.6 & 78.5 & 41.3 & 68.2 & 93.7 & 69.1 & 78.7 & 78.9 \\
   & Perplexity & 24.9 & 20.9 & 61.8 & 48.6 & 78.9 & 41.8 & 68.8 & 91.3 & 66.9 & 75.5 & 78.1 \\
   & GPS & 26.6 & 21.5 & 61.5 & 48.9 & 78.9 & 42.0 & 70.0 & 94.6 & 69.8 & 79.4 & 80.0 \\
   & APO & 27.1 & 22.1 & 61.5 & 49.3 & 79.4 & 42.3 & 70.3 & 94.8 & 70.1 & 79.9 & 79.7 \\
    & {\adv{}}  & \textbf{28.2}\textcolor[RGB]{84,123,71}{$\uparrow$$1.1$} &  \textbf{22.5}\textcolor[RGB]{84,123,71}{$\uparrow$$0.4$} & \textbf{63.6}\textcolor[RGB]{84,123,71}{$\uparrow$$1.8$} & \textbf{51.1}\textcolor[RGB]{84,123,71}{$\uparrow$$1.8$} & \textbf{80.4}\textcolor[RGB]{84,123,71}{$\uparrow$$1.0$}  & \textbf{43.2}\textcolor[RGB]{84,123,71}{$\uparrow$$0.9$} & \textbf{71.9}\textcolor[RGB]{84,123,71}{$\uparrow$$0.6$} & \textbf{95.8}\textcolor[RGB]{84,123,71}{$\uparrow$$1.0$} & \textbf{71.9}\textcolor[RGB]{84,123,71}{$\uparrow$$1.8$} & \textbf{82.3} \textcolor[RGB]{84,123,71}{$\uparrow$$2.4$} & \textbf{81.1} \textcolor[RGB]{84,123,71}{$\uparrow$$1.1$} \\
    \bottomrule
\end{tabular}}
  
\caption{
Main experimental results on generation, classification and reasoning tasks. Details of the selected few-shot prompts and the baselines are described in \Cref{subsec:experimental-setup}.
}
  \label{table:result_gen}
\end{table*}

\paragraph{Generation Tasks.}  
As shown in \Cref{table:result_gen},
{\adv{}} significantly outperforms all baseline methods across all backbone models, achieving $2.3\%, 2.9\%, 1.2\%$ average absolute improvements for text-davinci-002, Vicuna, and ChatGPT respectively. We observe that \adv{} achieves the most significant improvements for Summarization and Data-to-Text. Specifically, for \emph{text-davinci-002}, \adv outperforms the best baseline by 3.8\% on XSUM and 3.1\% on WebNLG. For Vicuna v1.5, \adv achieves an improvement of 5.6\% on the two data-to-text generation tasks WebNLG and E2E NLG. For ChatGPT, we achieve an improvement of 3.0\% on XSUM and 2.8\% on the E2E NLG generation task when compared to the vanilla few-shot baseline with no prompt optimization applied. We hypothesize that ChatGPT may obtain smaller absolute improvements when compared to other prompt optimization methods due to the misalignment between the backbone models of the generator and the discriminator. However, given that ChatGPT is the most widely used LLM and undergoes constant upgrades, it should be expected that improving ChatGPT is more difficult. 




\paragraph{Classification Tasks.} 
For classification tasks, {\adv{}} also brings significant improvements over all SOTA prompt optimization techniques across all models with $4.0\%, 2.9\%, 0.8\%$ average absolute improvements respectively. The most significant performance improvement is obtained using the \texttt{text-davinci-002} backbone. 
The 2.9\% improvements with Vicuna also illustrate the effectiveness of our proposed method on open-sourced models. 
The improvements of the three backbone models on classification tasks are relatively balanced. 

\paragraph{Reasoning Tasks.} For reasoning tasks, we observe a $2.7\%$ and $2.0\%$ absolute improvement on GSM8K and SVAMP, with text-davinci-002. Likewise, significant gains are observed with ChatGPT, achieving a $2.4\%$ increase on GSM8K and a $1.1\%$ boost on SVAMP. In the case of Vicuna, it achieves $3.2\%$ absolute improvement on GSM8K and $3.3\%$ absolute improvement on SVAMP. The effectiveness of \adv{} for reasoning tasks, particularly when coupled with CoT prompting, where the prompt includes detailed intermediate reasoning steps, demonstrates its ability to optimize complex prompts. This hints at potential for applying \adv{} to more advanced prompting methods.

\begin{table}[H]
\footnotesize
  \centering
  \setlength{\tabcolsep}{2.0pt}
  \scalebox{.75}{
  \begin{tabular}{lcccccc}
  \toprule
     & \textbf{Method} & \textbf{Humanity} & \textbf{STEM} & \textbf{Scocial Sciences} & \textbf{Others} & \textbf{Avg} \\ \midrule
\multirow{6}*{\rotatebox{90}{Vicuna v1.5}} & Few-shot &  $55.8$ & $38.7$ & $63.3$ & $61.5$ & $54.6$  \\ 
    & ROUGE-L & $55.5$ & $39.5$ & $63.7$ & $61.1$ & $55.0$ \\
& Perplexity & $55.2$ & $39.5$ & $64.1$ & $61.9$ & $55.2$ \\
 & GPS & $56.9$ & $40.4$ & $64.1$ & $62.3$ & $55.9$ \\
& APO & $57.2$ & $40.0$ & $63.7$ & $62.7$ & $55.9$ \\
  & {\adv{}} & $\mathbf{58.9}$ \textcolor[RGB]{84,123,71}{$\uparrow$$1.7$} & $\mathbf{44.1}$ \textcolor[RGB]{84,123,71}{$\uparrow$$3.7$} & $\mathbf{64.8}$ \textcolor[RGB]{84,123,71}{$\uparrow$$0.7$} & $\mathbf{64.5}$ \textcolor[RGB]{84,123,71}{$\uparrow$$1.8$} & $\mathbf{58.1}$ \textcolor[RGB]{84,123,71}{$\uparrow$$2.2$} \\
  \midrule
 \multirow{6}*{\rotatebox{90}{ChatGPT}} & Few-shot  & $73.9$  & $57.5$ & $79.2$ & $73.5$ & $71.0$  \\
    & ROUGE-L & $74.2$ & $56.7$ & $78.4$ & $73.9$ & $70.8$ \\
& Perplexity & $74.8$ & $56.3$ & $79.6$ & $71.2$ & $70.5$ \\
 & GPS & $74.6$ & $57.9$ & $80.0$ & $74.3$ & $71.7$ \\
& APO & $75.6$ & $58.3$ & $80.7$ & $73.9$ & $72.1$ \\
  & {\adv{}} & $\mathbf{76.7}$ \textcolor[RGB]{84,123,71}{$\uparrow$$1.1$} & $\mathbf{61.3}$ \textcolor[RGB]{84,123,71}{$\uparrow$$3.0$} & $\mathbf{82.3}$ \textcolor[RGB]{84,123,71}{$\uparrow$$1.6$} & $\mathbf{75.8}$ \textcolor[RGB]{84,123,71}{$\uparrow$$1.5$} & $\mathbf{74.0}$ \textcolor[RGB]{84,123,71}{$\uparrow$$1.9$} \\
    \bottomrule
  \end{tabular}}
\caption{Results of ChatGPT using 5-shot prompts on MMLU.
}
\label{tab:chat_mmlu}
\end{table}

\begin{figure}[ht]
    \centering
    \includegraphics[width=\linewidth]{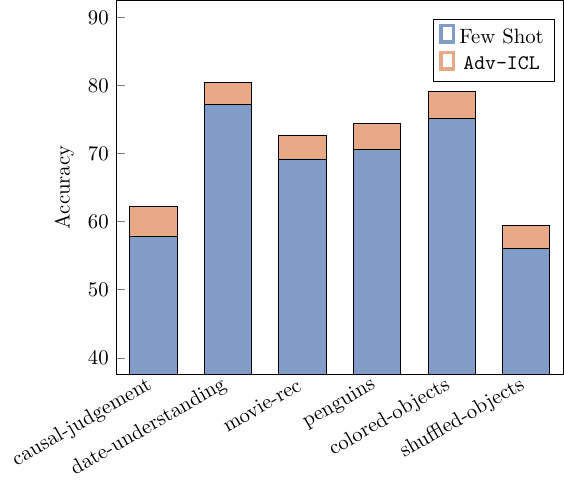}
    \caption{Results on selected tasks from BBH with ChatGPT using 5-shot Chain-of-Thought prompting. Full results can be found in \Cref{appendix:bbh_full} \
    }
    \label{fig:chat_bbh}
\end{figure}

\paragraph{MMLU \& BBH.} We summarize the results on MMLU in Table \ref{tab:chat_mmlu}. We improve the average performance from 69.8\% to 73.1\%, achieving performance improvements on 51 subjects out of 57 subjects with ChatGPT. 
For BBH, as shown in Figure \ref{fig:chat_bbh}, \adv{} achieves an accuracy of 70.6\% where the baseline method achieves an accuracy of 68.2\% with ChatGPT and chain-of-thought prompting. The detailed results on MMLU and BBH are in Appendix \ref{sec:apen_mmlu}.
Note that for BBH, only three data examples are provided with the dataset. Consequently, we use the same three examples as the initial data for both the generator and discriminator. Additionally, these 3 examples are the only real data examples utilized when estimating the objective. Despite this, we achieve substantial improvements on this task. This demonstrates the broad applicability of our method. In real-world scenarios with limited access to training samples our approach can still be effectively applied.

\section{Analysis} \label{sec:further-study}
In this section, we examine several design choices of \adv{}. {We further discuss the necessity of the discriminator in \Cref{sec:why-discriminator-works}, as well as an extended set of analyses in \Cref{appendix:extra-analysis}}.

\paragraph{Optimizing Instruction / Demonstration Only.} As instruction and demonstration data are both widely used in prompts, we examine the importance of optimizing both components. We use ChatGPT and compare our method with the prompt optimization method APE~\citep{zhou2023large}. We measure performance on WebNLG, GSM8K (with CoT), and MMLU. As shown in \Cref{fig:update_only_1_part}, we find that updating only the instruction or demonstrations makes the model perform suboptimally. Additionally, optimizing demonstrations is more effective than optimizing instructions for WebNLG and MMLU while the reverse is true for GSM8k. We hypothesize that this is because generated reasoning chains may contain errors and the correctness of the generated answers with respect to questions is critical for the model's performance \citep{min-etal-2022-rethinking}. That said, \adv{} achieves significant performance improvements for GSM8k in both cases.

 \begin{figure}[ht]
    \centering
    \includegraphics[width=\linewidth]{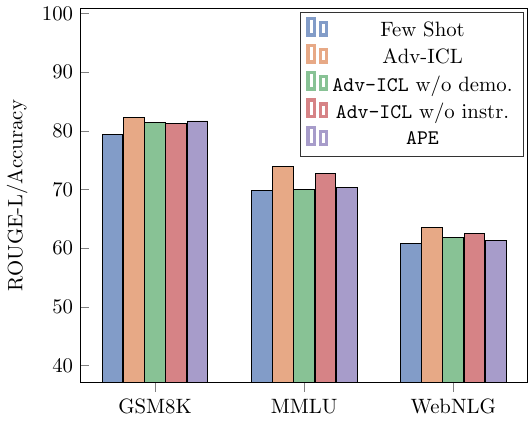}
    \caption{Ablation study on ChatGPT with {\adv{}} in which we only update the task instruction or demonstrations. 
    }
    \label{fig:update_only_1_part}
\end{figure}


\paragraph{Choosing Different Models for the Discriminator and Generator.} Given a generator, it is also important to answer how we can select a suitable discriminator to deploy our framework. In our main experiments, we chose the same model as the generator for all the base models, except ChatGPT. We hypothesize that since both discriminator and generator compete with each other in {\adv{}}, it is essential to balance their learning. To understand more whether we can use a discriminator different from the generator, we conducted experiments in \Cref{table:chatgpt_vicuna} dividing into two groups of using a stronger generator and a stronger discriminator in \adv{}. We observe with a stronger generator the performance is likely improved contrasting to with a stronger discriminator, the performance is potentially harmed. Overall, we suggest that the discriminator and generator should be chosen such that they are on the same performing level. A significant difference in their performance can drastically lower the overall framework's outcome. 


\paragraph{Ablation Studies on Number of Iterations and Data Samples.}

As shown in \Cref{algo:adversarial}, \adv involves three main hyperparameters: the number of training iterations $T$, the number of data points used per iteration $m$, and the number of new versions sampled for each instruction/demonstration $r$. We fix $r=5$ and analyze the best performing combination of $T$ and $m$ using grid search for $T\in\{1,3,5\}$ and $m\in\{1,2,5,10\}$. 


We measure \adv's performance on a validation set $S$ constructed from one representative task per category: WebNLG for generation, GSM8K for reasoning, and MMLU for classification. We use $80$ data samples from both WebNLG and GSM8K\footnote{GSM8K does not come with a validation set, so we sample from the training set instead.}. For MMLU, we sample $16, 16, 17, 19$ from the validation sets of \texttt{abstract\_algebra, business\_ethics, econometrics, formal\_logic} resulting in $228$ total samples in $S$.

We conduct experiments with ChatGPT and Vicuna as the backbone models. As shown in \Cref{table:hyper-parameter-search}, we observe the best performance with $T=3$ and $m=5$ for both settings. This demonstrates that our method works effectively without requiring many training iterations and data samples. \textcolor{black}{We further supplement $S$ with summarization and translation tasks ($80$ samples each) to conduct additional hyperparameter search experiments with Vicuna, as detailed in Appendix-\Cref{table:hyper-parameter-search-extra}. The results also indicate that the optimal settings are $T=3$ and $m=5$.} We provide an explanation of why training with too large $T$ or $m$ might harm model performance in \Cref{appendix:extra-analysis}.

\begin{table}
  \centering
  \begin{subtable}{.45\textwidth}
    \centering
    \scalebox{.65}{
      \begin{tabular}{lcccc}
        \toprule
        $m$ \textbackslash \ $T$ & $T=1$ & $T=3$ &  $T=5$ \\
        \midrule
$m=1$ & 61.3 / 78.8 / 42.6 & 63.8 / 80.0 / 47.1 & 62.5 / 80.0 / 48.5 \\ 
$m=3$ & 62.5 / 81.3 / 45.6 & 65.0 / 81.3 / 52.9 & 62.5 / 76.3 / 50.0 \\
$m=5$ & 63.8 / 82.5 / 54.4 & \textbf{66.3 / 82.5 / 55.9} & 63.8 / 77.5 / 54.4 \\
$m=10$ & 60.0 / 80.0 / 51.5 & 62.5 / 81.3 / 51.5 & 63.8 / 78.8 / 47.1 \\
    \bottomrule

    \end{tabular}
    }
    \caption{ChatGPT as $G$,  text-davinci-002 as $D$.}
    \label{tab:chatgpt-davinci2}
  \end{subtable}
  \hfill
  \begin{subtable}{.45\textwidth}
    \centering
    \scalebox{.65}{
      \begin{tabular}{lcccc}
        \toprule
        $m$ \textbackslash \ $T$ & $T=1$ & $T=3$ &  $T=5$ \\
        \midrule

$m=1$ & 52.5 / 40.0 / 50.0 & 53.8 / 43.8 / 55.9 & 53.8 / 42.5 / 54.4 \\
$m=3$ & 55.0 / 42.5 / 48.5 & 60.0 / 43.8 / 54.4 & 57.5 / 45.0 / 51.5 \\
$m=5$ & 55.0 / 41.4 / 48.5 & \textbf{61.3 / 45.0 / 54.4} & 57.5 / 42.5 / 51.5 \\
$m=10$ & 53.8 / 42.5 / 52.9 & 55.0 / 42.5 / 50.0 & 55.0 / 41.3 / 45.6 \\
\bottomrule
\end{tabular}
}
\caption{Vicuna as $G$,  Vicuna as $D$. }
\label{tab:vicuna-vicuna}
\end{subtable}
\caption{Ablation studies on number of iterations $T$ and number of samples used per iteration $m$. The results are ROUGE-L / Acc / Acc scores on WebNLG / GSM8K / MMLU.}
\label{table:hyper-parameter-search}
\end{table}

\paragraph{Qualitative Analysis.} 
For an intuitive understanding of how our prompt optimization progresses, we show examples of prompts for WebNLG changing over iterations in \Cref{appendix:extra-analysis}. The prompt modifier significantly alters the generator's prompt in two iterations by simplifying the instruction and adding a more specific requirement. The demonstrations are either replaced with a completely new one or are refined.

\section{Related Work}
\label{sec:related}
\paragraph{Adversarial Training.} 
Adversarial training has been widely used in image generation~\citep{goodfellow2014generative, radford2015unsupervised, arjovsky2017wasserstein}, domain adaptation~\citep{ganin2016domain, tzeng2017adversarial, xie2017controllable, louppe2017learning}, and improving model robustness~\citep{szegedy2013intriguing,biggio2013evasion,carlini2017towards,madry2018towards}. However, previous work shows that it often harms generalization \citep{raghunathan2019adversarial,min2021curious}. In NLP, there is an increasing interest in adversarial training; however, most current research focuses on its effect on generalization \citep{cheng-etal-2019-robust,wang2019improving,jiang2020smart}, and finetunes models \citep{jin2020bert,liu2020adversarial}, which is impractical for LLMs.  In contrast, \adv{} optimizes prompts and demonstrates strong generalization under different conditions.

\paragraph{Prompt Optimization.} In-context learning~\citep{liu2023pre} has sparked interest in prompt optimization (PO) techniques \citep{qin2021learning, deng2022rlprompt, lu2022dynamic, xu-etal-2022-gps, pryzant2023automatic, yang2023large, wang2024promptagent} for enhancing the performance of large language models (LLMs). Previous PO works fall into two categories: (1) continuous prompts; and (2) discrete textual prompts. Notable works optimizing continuous prompts include \citep{qin2021learning, liu2021gpt, lester2021power}. However, as model sizes increase, this approach becomes more computationally expensive. For very large language models, \citet{xu-etal-2022-gps} propose Genetic Prompt Search (GPS), a gradient-free prompt optimization method. Additionally, \citet{pryzant2023automatic} introduce Automatic Prompt Optimization (APO), utilizing text ``gradients" to evaluate and modify prompts. We compare {\adv{}} with GPS and APO. Other techniques like Automatic Prompt Engineer~\citep{zhou2023large} optimize only task instructions. We compare this with a variant of {\adv{}}. RL-based prompt optimization baselines like \citep{deng2022rlprompt, lu2022dynamic} are excluded due to involving additional MLP training and lacking a universal reward. Finally, PO algorithms have been recently developed to defend against jailbreaking attacks, for example, \cite{zou2023universal, zhu2023autodan, zhou2024robust}, but use different problem settings and are not directly comparable. \textcolor{black}{Finally, our method is indirectly related to demonstration selection works for ICL \cite{li-etal-2023-unified,yang-etal-2023-representative}. Since \adv{} does not focus on selecting the optimal combination and order of demonstrations, we do not compare our method with these approaches. However, demonstration selection algorithms can be developed concurrently and integrated with our method.}

\section{Conclusion}
In this work, we introduce \adv{}, an adversarial training framework for in-context learning using large language models. Our method has demonstrated empirical success across a diverse range of tasks and outperforms previous SOTA prompt optimization methods significantly. Effective with limited data samples and a very small number of training iterations, {\adv{}} holds promise for a wide array of real-world applications.

\section*{Limitations}
One limitation of our work is that \adv{} requires the component LLMs to follow human instructions well in performing their subtasks. However, we foresee that this limitation is going to be tackled by cutting-edge LLMs in the present and near future as LLMs are going to be more powerful. Additionally, choosing a good combination of \{Discriminator, Generator\} may require empirical experiments. In this work, we suggest that the same model can be used as both Discriminator and Generator. This offers strong performance as observed because both models are going to learn together well. However, in reality, many closed-source models like ChatGPT can be used as the Generator, but not the Discriminator. Choosing an optimal Discriminator in these cases requires deeper understanding as well as empirical experiments. We leave this exploration for future works.

\section*{Ethical Considerations} 
It is possible that this method could be used to optimize prompts for harmful purposes such as mis/disinformation generation, hatespeech, or privacy violating use cases. While this is not what our method is designed for, there is no way to prevent this type of misuse. While our method could also improve the efficiency and efficacy of bad actors, we do not anticipate that there is anything inherent to \adv allowing it to be more effective in these settings than in other, positive, settings. 

\section*{Acknowledgements}
\textcolor{black}{This research is partially supported by the National Research Foundation Singapore under the AI Singapore Programme (AISG Award No: AISG2-TC-2023-010-SGIL) and the Singapore Ministry of Education Academic Research Fund Tier 1 (Award No: T1 251RES2207). Do Xuan Long is supported by the A*STAR Computing and Information Science (ACIS) scholarship. We thank the anonymous reviewers for the constructive and helpful feedback.}

\bibliography{iclr2024_conference}
\bibliographystyle{iclr2024_conference}

\appendix

\section{Appendix}


\subsection{Theoretical Proofs of the Convergence} \label{sec:theoretical-proof}
In this section, we theoretically analyze whether such a minimax objective in the form of in-context learning can achieve the desired equilibrium as in the original GAN scenario. We assume access to models with infinite capacities powering the discriminator $D$, generator $G$, and prompt modifier $M$ and that in each iteration, we sample a sufficient number of prompts from $M$ to update both $G$ and $D$. Let $p_{data}$ be the distribution of the training data, and $p_g$ be the distribution of the generated data from $G$. 

Considering a language model which can be $D$ or $G$ with powerful enough in-context learning capabilities, given a task,  we further assume that:

\begin{enumerate}
  \item $M$ is powerful enough to modify the initial prompt of $D/G$, covering all possible prompt variants.
  \item There exists a prompt $\mathcal{P}$ for $D/G$ that given $\mathcal{P}$, $D/G$ can achieve the globally optimal result.
  \item $M$ can generate $\mathcal{P}$ by which $D/G$  achieves the globally optimal result.
\end{enumerate}

The assumption 3 is a result of assumptions 1, and 2, and the assumption about our access to infinite capacities language models. Indeed, given $D/G$, from assumption 2, there exists a globally optimized prompt $\mathcal{P}$ for it such that it can achieve the globally optimal state for a given task. Furthermore, since $M$ is powerful enough in modifying the initial prompt (ass. 1), plus $M$ samples a sufficiently large number of prompts for each iteration (ass. 2), $M$ can generate $\mathcal{P}$ with a non-zero probability, which concludes the assumption 3.

With the above assumptions, we prove the following results.

\begin{prop}{\citep{goodfellow2014generative}}\label{prop:prop1}
For $G$ fixed, the optimal discriminator $D$ can be described in a closed form, denoted as $D^*$.
\end{prop}

\begin{proof}[Proof for Proposition~\ref{prop:prop1}]
Adapted from \citep{goodfellow2014generative}: For a fixed $G$, the training objective for the discriminator $D$ is maximizing the adversarial loss $\mathcal{J}(D, G)$ (\Cref{eq:loss})


\[
\begin{split}
  \mathcal{J}(D, G) &=  \!\begin{multlined}[t]
       \mathbb{E}_{x, y\sim p_{data}} \log \Big( D(x, y) \Big) \\ + \mathbb{E}_{x\sim p_{data}} \log\Big( 1 - D\big(x, G(x)\big) \Big) 
     \end{multlined} \\
    &= \!\begin{multlined}[t]
       \mathbb{E}_{x, y\sim p_{data}} \log \Big( D(x, y) \Big) \\ + \mathbb{E}_{x,y\sim p_{g}} \log\Big( 1 - D\big(x, y\big) \Big)
     \end{multlined} \\
    &= \!\begin{multlined}[t]
       \int_{x} p_{data}(x) \log D(x, y) \Big) \,dx \\ + \int_{x} p_{g}(x) \log \Big( 1 - D\big(x, y\big) \Big) \,dx
     \end{multlined} \\
    &= \!\begin{multlined}[t]
       \int_{x} p_{data}(x) \log D(x, y) \\ +  p_g(x) \log \Big( 1 - D\big(x, y\big) \Big) \,dx
     \end{multlined}
   \end{split}
\]

The function $y = a\log(x) + b\log(1-x)$ for $(a,b) \in \mathbb{R}^2 and (a,b) \neq \{0,0\}$ achieves its maximum in $[0, 1]$ at $\frac{a}{a + b}$. Therefore, $D^*(x)$ has a closed form, which is $D^*(x) = \frac{p_{data}(x)}{p_{data}(x) + p_{g}(x)}$.

\end{proof}



\begin{prop}{(Motivated by \cite{goodfellow2014generative})} \label{prop:prop4}
If $G$ and $D$ have enough capacity, and at each training step, the discriminator is allowed to reach its optimum $D^*$ given $G$, and $p_g$ is updated so as to improve the criterion 

\begin{equation}
\begin{multlined}
\mathcal{J}(D^*, G)=\mathbb{E}_{x, y\sim p_{data}} \log \Big( D^*(x, y) \Big) \\ + \mathbb{E}_{x\sim p_{data}} \log\Big( 1 - D^*\big(x, G(x)\big) \Big) 
\end{multlined}
\end{equation}

then $p_g$ converges to $p_{data}$.
\end{prop}

\begin{proof}[Proof for Proposition~\ref{prop:prop4}]
At each training step, the optimal $D^*$ can be achieved via editing its input prompt by $M$. Considering the loss function $\mathcal{J}(D^*, G)$ as a function in $p_g$, then $\mathcal{J}(D^*, G)$ is convex in $p_g$. Since $G$ is powerful enough that there exists a prompt $\mathcal{P}$ sampled by $M$ such that $G$ can achieve the globally optimal loss $\mathcal{J}$ (assumption 2), with an optimal $D^*$, we can obtain the corresponding best $G$. Furthermore, $\mathcal{J}(D^*, G)$ is convex in $p_g$, plus the global optimal of $G$ can be obtained, with a sufficiently large enough number of prompts sampled and training iterations, $p_g$ converges to $p_{data}$.
\end{proof}


\subsection{Baseline Implementation} \label{appendix:baseline-implementation}
In this section, we present our implementation details for the baselines. First, among the benchmarks we used, the following datasets do not have any validation set with sizes larger than or equal to $80$: YELP, WSC, GSM8K, SVAPM. Therefore, we randomly sample $100$ data cases from their training sets, to create their validation sets. 

Each baseline requires a development set to decide which prompt(s) is/are the best at each optimization iteration. For GPS and APO, we sample $32$ and $50$ data samples respectively from the validation set of each benchmark, following \cite{xu-etal-2022-gps, pryzant2023automatic}. For ROUGE-L and Perplexity, we sample $80$ data samples, also from each validation set.  Additionally, among the baselines, only APO requires training data for error messages. For a fair comparison with \adv{}, we use the same training data samples with {\adv{}} as training data for APO. 

\paragraph{ROUGE-L \& Perplexity \citep{gonen2022demystifying}.} For these baselines, we utilize ROUGE-L \citep{lin2004rouge} or Perplexity \citet{gonen2022demystifying} as the measurement to optimize the input instruction and demonstrations sequentially.
For the instruction, we sample $15$ new instructions by paraphrasing following the template: \texttt{'Write for me 15 paraphrases of the \{initial\_instruction\}:'}. We then select the version which achieves the best result on $S$ as the final instruction. Similarly, for each demonstration, we use the template \texttt{'Write 15 paraphrases for the following example. Keep the format as Input: and Output:. End the answer by So the answer is:'} to sample $15$ versions of the original demonstrations, and select the best one on $S$ sequentially until all the demonstrations are optimized.  We sample $15$ versions for comparisons because our proposed {\adv{}} also samples a maximum of $15$ versions for the instruction and each demonstration. 

\paragraph{GPS \citep{xu-etal-2022-gps}.} We run GPS \citep{xu-etal-2022-gps} on $3$ iterations to optimize the instruction and each demonstration sequentially. 
Denote the original instruction/demonstration to be optimized as $O$. 
In the initial step, given the original human-written $O$, we paraphrase it into $10$ versions using \texttt{'Write for me 10 paraphrases of the \{initial\_instruction\}:'} for instruction, and \texttt{'Write 10 paraphrases for the following example. Keep the format as Input: and Output:. End the answer with <END>. So the answer is:'} for demonstration. We then select the top-$5$ generated $O$ to pass to the first iteration. At each iteration, for each $O$ in the current top$-5$ $Os$, we sample $5$ new $Os$ by Sentence Continuation strategy \citep{schick-schutze-2021-generating} via using the backbone LLM itself, and select the top$-5$ $Os$ among $25$ $Os$ to the next iteration. Finally, the best-performing $O$ on $S$ is selected as the output instruction/demonstration of the method. It is worth noting that in the original paper from \cite{xu-etal-2022-gps}, top$-k$ with $k=25$ was used. However, in our reimplementation, we use $k=5$ so that it can be relatively fair to compare GPS with our method (we use $r=5$) and other baselines. The template for sampling new prompts via the Sentence Continuation strategy that we used is exactly the same as \citet{xu-etal-2022-gps} provided.

\paragraph{APO \citep{pryzant2023automatic}.} Since our setting assumes that we have access to limited training data samples, we reimplemented a simplified version of the original APO in which the selection step \citep{pryzant2023automatic} only be called once, and the samples that we used to train {\adv{}} are returned. For simplicity, we call the original instruction/demonstration as $O$. We run APO to optimize the instruction and each demonstration sequentially in a given prompt. Given an initial $O$, and the error samples, we use the backbone LLM to generate feedback consisting of $5$ comments as the text "gradient". Integrating this gradient as feedback, we ask the LLM to generate $10$ prompt samples. We further utilize the backbone LLM to generate $5$ paraphrase versions of the original $O$, resulting in a total of $15$ new $Os$. Finally, we select the best $O$ evaluated on $S$. All the prompt templates for generating gradients, integrating feedback, and generating paraphrased prompts are adopted from \cite{pryzant2023automatic}. For selecting error samples, in the original implementation, \citet{pryzant2023automatic} compared the generated answer with the ground-truth answer, and the error samples are the ones that have the generated answer different from the ground-truth answers. This is applicable for classification and numerical question-answering tasks, but not the text generation tasks such as summarization, this strategy of selecting error samples is not suitable. Therefore, for summarization, data-to-text, and translation tasks, we select one sample that the current prompt brings the lowest ROUGE-L score as the sole error sample. 

\paragraph{APE \citep{zhou2023large}.} For APE, we adopt the implementation on the GitHub\footnote{github.com/keirp/automatic\_prompt\_engineer/tree/main} from \cite{zhou2023large}. We limit the number of instructions sampled to $15$ to have fair comparisons with {\adv{}}. For the training samples for each task, we use the same samples that we train {\adv{}} for APE.

\subsection{Supplementary Experiment Details}
In this section, we provide more details used in the experiments.

\paragraph{Number of Demonstrations for Few-shot Experiments.} \label{appendix:number-of-shots}

Number of demonstrations for few-shot experiments of all datasets is listed in Table \ref{table:number_of_shots}. For generation tasks and classification tasks,  We follow the expert-written prompts from Super-NaturalInstruction~\citep{wang2022super}. For reasoning tasks, MMLU and BBH, we follow the standard prompts that they propose in their paper or open-source code.

\begin{table*}
  \centering
\footnotesize
  \scalebox{0.63}{
  \begin{tabular}{l|llllll|lll|ll|ll}
    \toprule
    \multirow{2}*{\textbf{{}}} & \multicolumn{2}{c}{\textbf{\normalsize{Summarization}}}  & \multicolumn{2}{c}{\textbf{\normalsize{Data-to-Text}}} & \multicolumn{2}{c}{\textbf{\normalsize{Translation}}} & \multicolumn{3}{c}{\textbf{\normalsize{Classification}}} & \multicolumn{2}{c}{\textbf{\normalsize{Reasoning}}} &
    \multicolumn{2}{c}{\textbf{\normalsize{Evaluation Suits}}} \\\specialrule{0em}{2pt}{2pt}
    & \textbf{XSUM} & \textbf{CNN} & \textbf{WebNLG} & \textbf{E2E NLG}  & \textbf{RO $\rightarrow$ EN} & \textbf{IT$\rightarrow$ JA} & \textbf{YELP Review}  & \textbf{COPA}  & \textbf{WSC} & \textbf{GSM8K} & \textbf{SVAMP} & \textbf{MMLU} & \textbf{BBH}\\
    \midrule
    {\#shots}  & 3 & 2 & 3 & 2 & 3 & 3 & 3 & 3 & 3 & 5 & 5 & 5 & 3 \\
    \bottomrule
  \end{tabular}}
\caption{\small{Number of shots used for \emph{few-shot} experiments.}}
  \label{table:number_of_shots}
\end{table*}

\paragraph{Test Set Statistics.} \label{appendix:benchmark_statistics}

As mentioned in the main paper, we sample a subset of the test set for efficient evaluation. In \Cref{table:benchmarks_statistics}, we show the exact numbers of testing samples we used for each task.

\begin{table*}
  \centering
\footnotesize
  \scalebox{0.7}{
  \begin{tabular}{l|llllll|lll|ll}
    \toprule
    \multirow{2}*{\textbf{{}}} & \multicolumn{2}{c}{\textbf{\normalsize{Summarization}}}  & \multicolumn{2}{c}{\textbf{\normalsize{Data-to-Text}}} & \multicolumn{2}{c}{\textbf{\normalsize{Translation}}} & \multicolumn{3}{c}{\textbf{\normalsize{Classification}}} & \multicolumn{2}{c}{\textbf{\normalsize{Reasoning}}} \\\specialrule{0em}{2pt}{2pt}
    & XSUM & CNN & WebNLG & E2E NLG  & RO $\rightarrow$ EN & IT$\rightarrow$ JA & \textbf{YELP Review}  & \textbf{COPA}  & \textbf{WSC} & \textbf{GSM8K} & \textbf{SVAMP}\\
    \midrule
    {\#test samples}  & $1000$ & $950$ & $1000$ & $1000$ & $1000$ & $1000$ &  $1000$& $496$ & $285$ & $1319$ & $1000$ \\
    \bottomrule
  \end{tabular}}
  
\caption{\small{Test set statistics.}}
\label{table:benchmarks_statistics}
\end{table*}

\paragraph{Prompt Modifier Prompts.} \label{appendix:prompts-for-prompt-modifier} Here, we also provide the prompt used in the prompt modifier. The prompt is as follows:

\begin{itemize}
  \item Modifying instructions: \texttt{Generate 5 variations of the following instruction while keeping the semantic meaning. Keep the generated instructions as declarative. Wrap each with <START> and <END>.}.
  \item Modifying open-ended QA pairs: \texttt{Generate 5 variations of the following example to make them more representative. Keep the format as Input: and Output:. Wrap each with <START> and <END>.}.
  \item Modifying MCQ pairs: \texttt{Generate 5 variations of the following multiple-choice question and the answer to make them more representative. Keep the format as multiple-choice question and the answer. Keep the format as Input: and Output:. Wrap each with <START> and <END>.}.
\end{itemize}

\paragraph{Extended Experimental Details.} \label{appendix:extended-experimental-details}
For OpenAI API models, ChatGPT (gpt-3.5-turbo-0613) with chat completion mode and text-davinci-002 with text completion mode were called at temperature 0.6. For open-source baselines, Vicuna v1.5 13B was used with a window size of 1024. We use Nucleus Sampling \cite{Holtzman2020The} as our decoding strategy for all the models with a p value of $0.9$.

\subsection{Why the Discriminator Works?} \label{sec:why-discriminator-works}
We further conduct experiments (\Cref{table:freezing-discriminator}) to verify whether the prompt modifier module work as expected. Specifically, we remove the discriminator and only employ a prompt modifier to repeatedly optimize the prompt.

\begin{table}[H]
  \centering
\footnotesize
  \scalebox{0.75}{
  \begin{tabular}{l|ccccc}
    \toprule
    & \textbf{WebNLG} & \textbf{RO} $\rightarrow$ \textbf{EN} & \textbf{YELP} & \textbf{GSM8K}\\
    \midrule
    {Vicuna 13B}  & 52.5 & 72.1 & 71.0 & 40.7 \\
    {\adv{}  \emph{w.o. discriminator}}  & 50.1 & 71.4 & 72.1 & 40.2 \\
    {\adv{}}  & 59.3 & 73.4 & 73.6 & 43.9 \\
    \midrule
    {ChatGPT}  & 60.9 & 78.8 & 69.8 & 79.4 \\
    {\adv{}  \emph{w.o. discriminator}}  & 61.2 & 77.4 & 64.5 & 71.6 \\
    {\adv{}}  & 63.6 & 80.4 & 71.9 & 82.3 \\
    \bottomrule
  \end{tabular}}
  
\caption{\small{Experimental results with Vicuna and ChatGPT with \adv{} when being removed the discriminator.}}
  \label{table:freezing-discriminator}
\vspace{-5mm}
\end{table}

In most cases, removing the discriminator and relying solely on the prompt modifier under Vicuna and ChatGPT leads to a decline in performance. This observation highlights the importance of the discriminator and adversarial loss in the optimization process.

\subsection{Extended Experiments} \label{appendix:extra-analysis}

\paragraph{Choosing Different Models for the Discriminator and Generator.} \Cref{table:chatgpt_vicuna} presents our experimental results.

\begin{table*}
  \centering
\footnotesize
  \scalebox{0.8}{
  \begin{tabular}{l|l|ccccc}
    \toprule
    Group & Models & \textbf{WebNLG} & \textbf{LIRO} & \textbf{YELP} & \textbf{GSM8K}\\
    \midrule
    & text-davinci-002 & 65.4 & 81.2 & 74.4 & 50.8 \\
    \adv{} & vicuna 13B & 59.3 & 73.4 & 73.6 & 43.9 \\
    & ChatGPT & 63.6 & 80.4 & 71.9 & 82.3 \\
    \midrule
    & {vicuna 7B (D) + vicuna 13B (G)}  & 61.1 & 72.9 & 72.4 &  41.9 \\
    & {vicuna 7B (D) + text-davinci-002 (G)}  & 62.3 & 77.9 & 71.2 & 44.1 \\
    Stronger Generator & {vicuna 7B (D) + ChatGPT (G)}  & 62.1 & 78.8 & 70.6 &  80.9  \\
    & {vicuna 13B (D) + text-davinci-002 (G)}  & 63.9 & 79.6 & 72.9 &  49.8  \\
    & {vicuna 13B (D) + ChatGPT (G)}  & 63.6 & 78.9 & 71.4 &  81.2  \\
    \midrule
    & {vicuna 13B (D) + vicuna 7B (G)}  & 58.9 & 73.3 & 63.6 &  22.3  \\
    Stronger Discriminator & {text-davinci-002 (D) + vicuna 7B (G)}  & 58.5 & 72.2 & 62.8 &  20.6  \\
    & {text-davinci-002 (D) + vicuna 13B (G)}  & 58.8 & 72.4 & 73.4 &  44.2  \\
    \bottomrule
  \end{tabular}
}  
\caption{\small{Experiments of using different discriminators and generators.}}
  \label{table:chatgpt_vicuna}
\end{table*}

\paragraph{Reliability of The Results.} We rerun our experiments with \adv{} three times on WebNLG, RO $\rightarrow$ EN, YELP, GSM8K. The results are presented in \Cref{table:result-reliability}.

\begin{table}[H]
  \centering
\footnotesize
  \scalebox{0.7}{
  \begin{tabular}{l|ccccc}
    \toprule
    & \textbf{WebNLG} & \textbf{RO} $\rightarrow$ \textbf{EN} & \textbf{YELP} & \textbf{GSM8K}\\
    \midrule
    {Vicuna 13B} &  59.3/59.2/59.5 & 73.4/74.1/73.2 & 73.6/73.6/73.5 & 43.9/44.3/44.1 \\
    {ChatGPT} & 63.6/63.5/63.8 & 80.4/80.6/80.6 & 71.9/71.8/71.9 & 82.3/82.5/82.2 \\
    \bottomrule
  \end{tabular}}
  
\caption{\small{Our experimental results with \adv{} on three different runs.}}
  \label{table:result-reliability}
\end{table}

The results clearly demonstrate that \adv{} consistently delivers stable outcomes, thereby highlighting its reliability in faithfully reproducing our experimental findings.

\paragraph{Providing More Feedback to the Prompt Modifier.}
We conducted an experiment that involved integrating the most successful prompts from previous iterations as feedback for the next iteration. In this process, we utilized previous best-performing prompts, namely $P_1, P_2, ..., P_k$, as inputs to the prompt constructor module in order to generate the $(k+1)$-th prompt, denoted as $\{P_1, ..., P_k\}$. The template for optimizing task instruction is shown as follows, similar to the prompt for optimizing demonstrations.

\texttt{Diversify the task instruction to be clearer. Keep the task instruction as declarative.}

\texttt{Task instruction: $P_0$}

\texttt{Improved task instruction: $P_1$}

\texttt{…}

\texttt{Task instruction: $P_{k-1}$}

\texttt{Improved task instruction: $P_k$}

\texttt{Task instruction: $P_k$}

\texttt{Improved task instruction:}

We applied the method to four representative tasks WebNLG, RO $\rightarrow$ EN, YELP, GSM8K using both Vicuna and ChatGPT models. The obtained results for are illustrated in \Cref{table:prompt-modifier-with-feedback}.

\begin{table}[H]
  \centering
\footnotesize
  \scalebox{0.65}{
  \begin{tabular}{l|ccccc}
    \toprule
    & \textbf{WebNLG} & \textbf{RO} $\rightarrow$ \textbf{EN} & \textbf{YELP} & \textbf{GSM8K}\\
    \midrule
    {Vicuna 13B}  & 52.5 & 72.1 & 71.0 & 40.7 \\
    {\adv{} (prompt modifier with history)}  & 56.9 & 74.0 & 74.2 & 42.2 \\
    {\adv{}}  & 59.3 & 73.4 & 73.6 & 43.9 \\
    \midrule
    {ChatGPT}  & 60.9 & 78.8 & 69.8 & 79.4 \\
    {\adv{} (prompt modifier with history)}  & 62.1 & 79.8 & 72.1 & 80.9 \\
    {\adv{}}  & 63.6 & 80.4 & 71.9 & 82.3 \\
    \bottomrule
  \end{tabular}}
\caption{\small{Experimental results with Vicuna and ChatGPT with the feedback to the prompt modifier.}}
  \label{table:prompt-modifier-with-feedback}
\end{table}

In the case of Vicuna, incorporating additional feedback into the prompt modifier proves effective for tasks such as translation and classification. However, this approach falls short when applied to data-to-text and reasoning tasks. On the other hand, for ChatGPT, augmenting the prompt modifier with more feedback does not yield improved performance. This can be attributed to ChatGPT's strong zero-shot prompt capabilities, which outshine its ability to perform effectively with few-shot prompts.

\paragraph{Ablation Studies on Number of Generated Samples $r$.} We investigate whether generating fewer / more samples in each prompt modification  would affect the model's performance. Due to the limited resources, we only conducted the experiment on the WebNLG and GSM8k dataset, with $r\in\{1,3,5,10,20\}$. The results are shown in \Cref{fig:ablation-r}. We observe that increasing $r$ lead to comparable results.

\paragraph{Why Might too Many Iterations $T$ or Samples $m$ Harm the Performance of Models?}
We observed this phenomenon in the experiments and were also curious about it. We hypothesize that first, training with too many iterations can cause the model to be overfitting to the task, leading to worse performance on the test samples. Second, \adv{}, a specialized form of in-context learning, plays a crucial role in enhancing the performance of LLMs by enabling them to learn from the training examples and generate improved prompts. While in-context learning holds great promise, it is essential to acknowledge that increasing the number of training examples does not necessarily guarantee better performance. As demonstrated by \cite{min-etal-2022-rethinking}, a critical threshold exists for the number of training examples, and surpassing this threshold leads to a decline in performance. Thus, in our specific settings, augmenting the training examples did not yield better results.

Given its inherent complexity and non-deterministic nature, we have put forward a hyper-parameter tuning approach, presented in \Cref{table:hyper-parameter-search}, aimed at determining these hyper-parameters for new configuration settings.

\paragraph{Prompt Modifier Temperature.} 
Lastly, we examine the influence of the generation temperature for the prompt modifier. Ideally, the prompt modifier should have enough diversity to generate potential improvements for the prompts of both the generator and discriminator. Intuitively, this means we should not use greedy decoding with a temperature of 0 for the prompt modifier. As demonstrated in Figure~\ref{fig:temperature}, a temperature of 0.6 works well, providing a sufficiently large search space while still generating high-quality prompts.

\begin{figure}[t]
\centering
\begin{minipage}[t]{0.48\textwidth}
\centering
\includegraphics[width=0.95\textwidth]{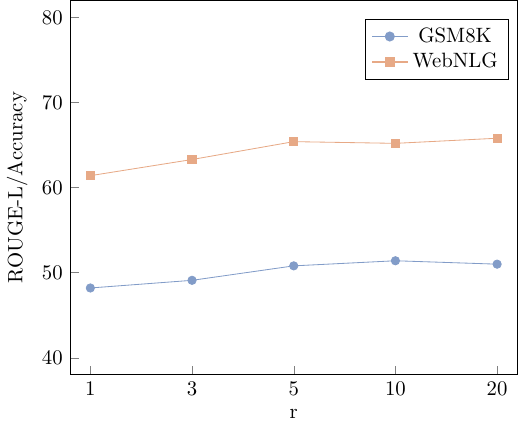}
\caption{Ablation study on the number of sampled prompt $r$.}
\label{fig:ablation-r}
\end{minipage}
\begin{minipage}[t]{0.48\textwidth}
\centering
\includegraphics[width=0.95\textwidth]{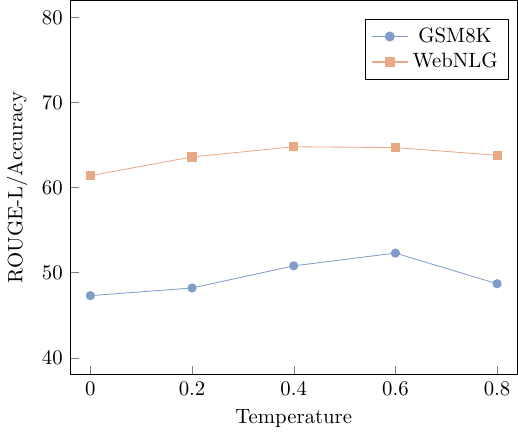}
\caption{Ablation study on temperature of the prompt modifier. }
\label{fig:temperature}
\end{minipage}
\end{figure}


Here we present the detailed results of human evaluation on generated instructions and  demonstrations respectively. Details are shown in Table \ref{table:human}. text-davinci-002 and ChatGPT achieve similar performance with the zero-shot prompt modifier, while Vicuna performs a little bit worse but also achieves an acceptable correctness ($\ge 80$).

\begin{table}[ht]
  \centering
  \scalebox{0.65}{
  \begin{tabular}{lcccc}
    \toprule
Model & 30 instructions & 70 demonstrations & Overall \\ \midrule
text-davinci-002 & 93.3 & 85.7 & 88.0 \\
Vicuna v1.5 & 90.0 & 80.0 & 83.0 \\
ChatGPT & \textbf{96.7} & \textbf{88.6} & \textbf{91.0} \\
    \bottomrule
  \end{tabular}}
  \caption{Human evaluation results for each specific type of modifications.}
  \label{table:human}
  
\end{table}

\paragraph{Using More Tasks To Select Hyperparameters.} 
\textcolor{black}{\Cref{table:hyper-parameter-search-extra} presents our hyperparameter selection experiments for Vicuna, expanding on the tasks included in \Cref{table:hyper-parameter-search} by adding summarization and translation tasks. We observe that ($T=3$, $m=5$) is also the best combination.} 

\begin{table*}
\centering
\begin{subtable}{1\textwidth}
\centering
\scalebox{.65}{
\begin{tabular}{lcccc}
\toprule
$m$ \textbackslash \ $T$ & $T=1$ & $T=3$ &  $T=5$ \\
\midrule
$m=1$ & 19.92 / 52.5 / 71.88 / 40.0 / 50.0 & 20.89 / 53.8 / 72.02 / 43.8 / 55.9 & 21.01 / 53.8 / 72.66 / 42.5 / 54.4 \\
$m=3$ & 20.68 / 55.0 / 70.15 / 42.5 / 48.5 & 21.11 / 60.0 / 71.89 / 43.8 / 54.4 & 21.80 / 57.5 / 72.15 / 45.0 / 51.5 \\
$m=5$ & 22.20 / 55.0 / 72.78 / 41.4 / 48.5 & \textbf{22.18 / 61.3 / 73.06 / 45.0 / 54.4} & 21.88 / 57.5 / 72.44 / 42.5 / 51.5 \\
$m=10$ & 18.89 / 53.8 / 66.82 / 42.5 / 52.9 & 19.03 / 55.0 / 68.90 / 42.5 / 50.0 & 19.90 / 55.0 / 71.56 / 41.3 / 45.6 \\
\bottomrule
\end{tabular}
}
\caption{Vicuna as $G$,  Vicuna as $D$. }
\label{tab:vicuna-vicuna}
\end{subtable}
\caption{Ablation studies on number of iterations $T$ and number of samples used per iteration $m$. The results are ROUGE-L / Acc / Acc scores on XSUM / WebNLG / LIRO / GSM8K / MMLU.}
\label{table:hyper-parameter-search-extra}
\end{table*}

\paragraph{More Qualitative Analysis.}
We show one example in \Cref{fig:qualitative_analysis}. We also show an additional case of qualitative analysis on Yelp. As shown in \ref{fig:example_cla}, the optimization follows a similar pattern with that on the data-to-text task. 

\begin{figure*}
    \centering
\includegraphics[width=.95\textwidth]{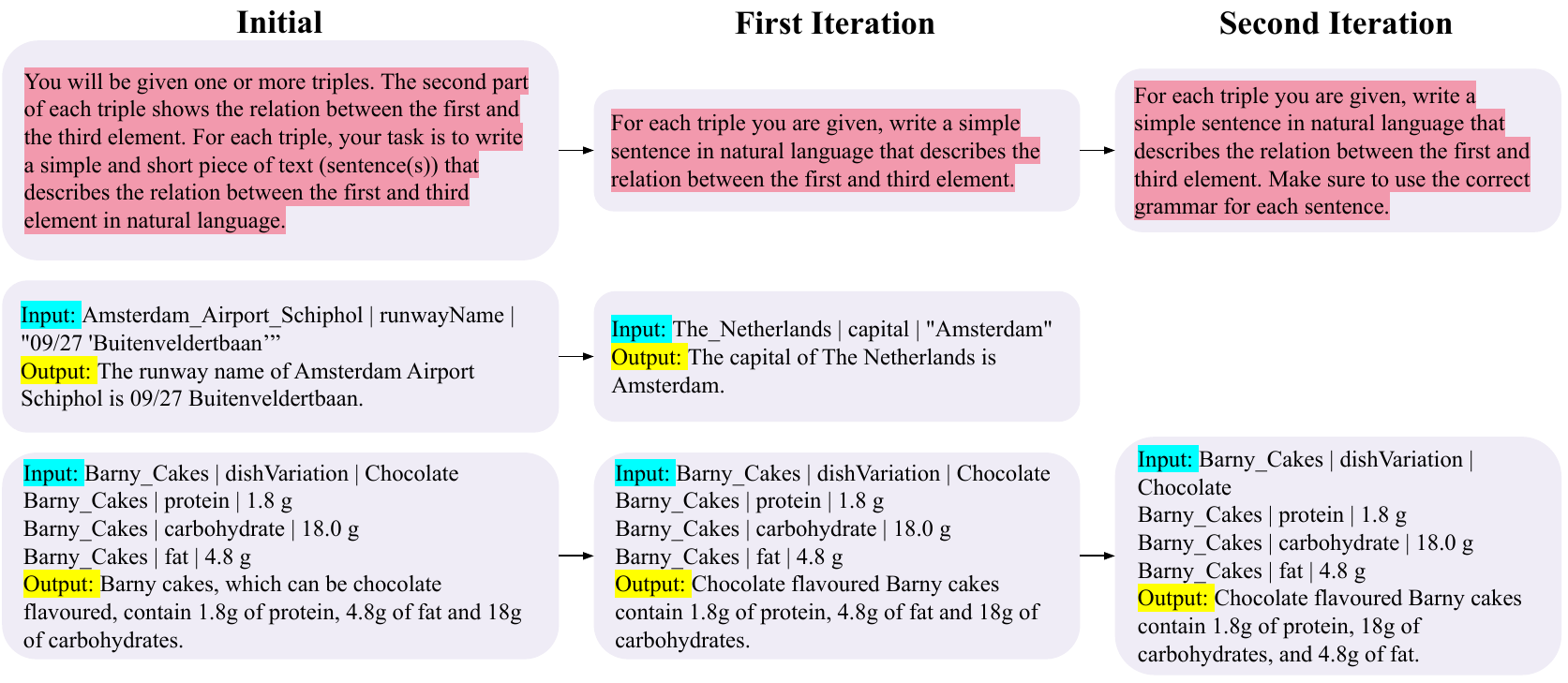}
    \caption{Optimization for the prompt on the data-to-text task WebNLG.}
\label{fig:qualitative_analysis}
\end{figure*}

\begin{figure*}
\centering
\includegraphics[width=1\textwidth]{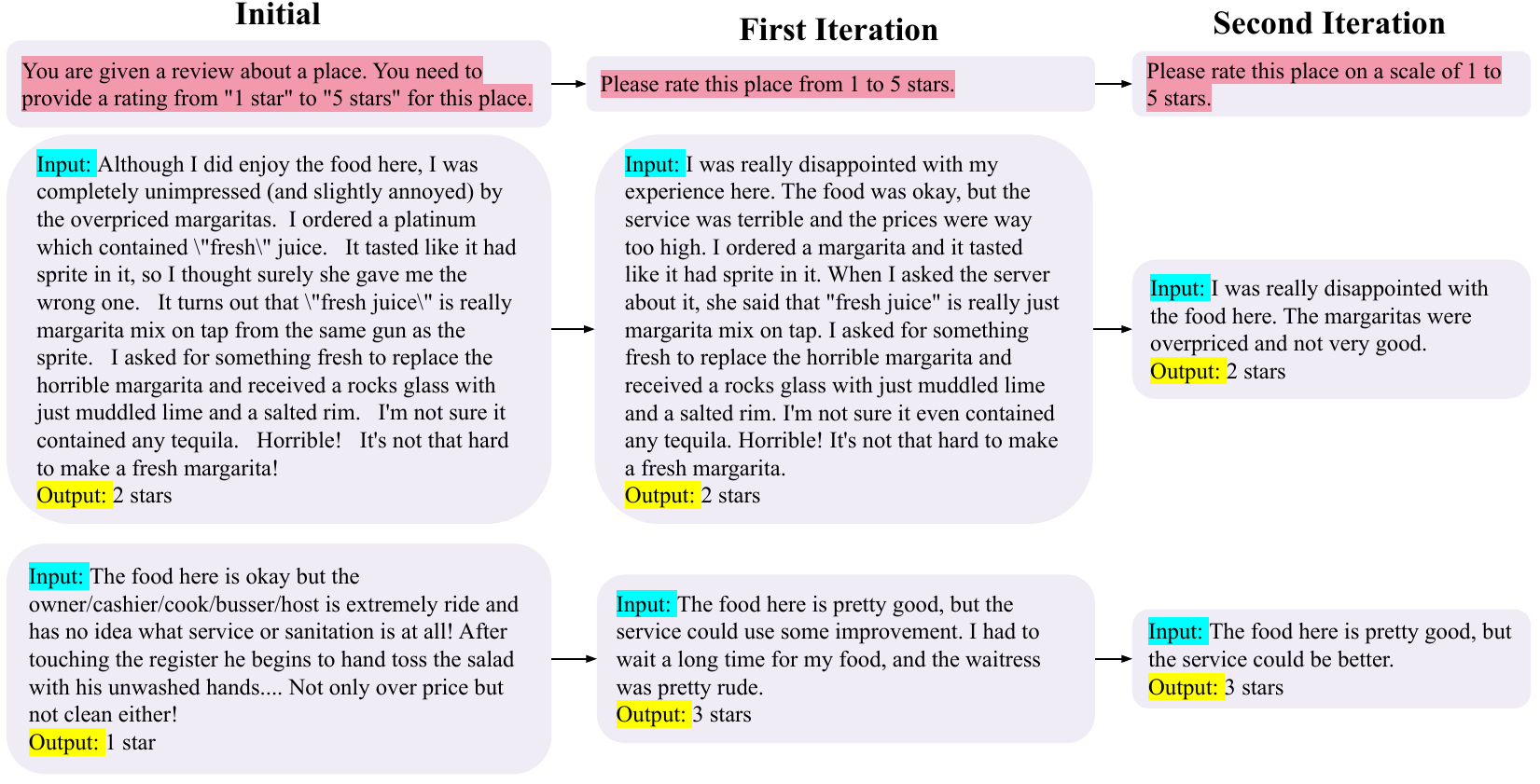}
\caption{Qualitative analysis on the classification task Yelp.}
\label{fig:example_cla}
\end{figure*}

\paragraph{Detailed Results on MMLU.}\label{sec:apen_mmlu}
In Figure \ref{fig:chat_mmlu}, we show the detailed results on MMLU with ChatGPT. As shown in the graph, \adv{} achieves significant improvements on most tasks. 

\begin{figure*}
    \centering
    \includegraphics[width=\linewidth]{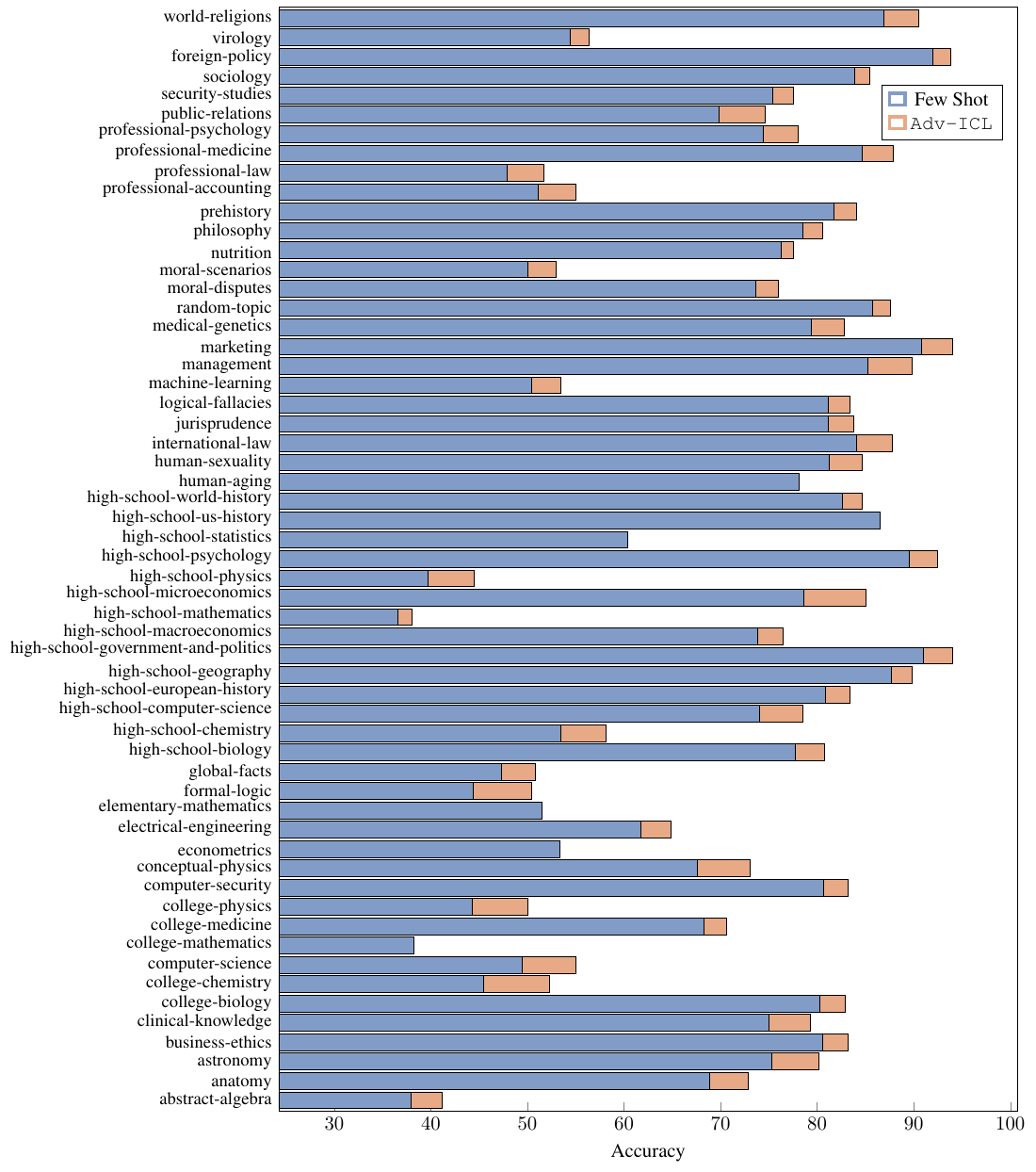}
\caption{Results on MMLU using ChatGPT, where the y-axis begins at $25\%$, representing the baseline of random choices.}
\label{fig:chat_mmlu}
\end{figure*}

\paragraph{Detailed Results on BBH.}\label{appendix:bbh_full}
In \Cref{fig:bbh_full}, we show the full results of ChatGPT on BIG-Bench Hard using 5-shot Chain-of-Thought prompting. The baseline achieves an average of 68.2\% accuracy while \adv{} reaches an average of accuracy of 70.6\% and never performs worse than the baseline.

\begin{figure*}
    \centering
\includegraphics[width=0.7\textwidth]{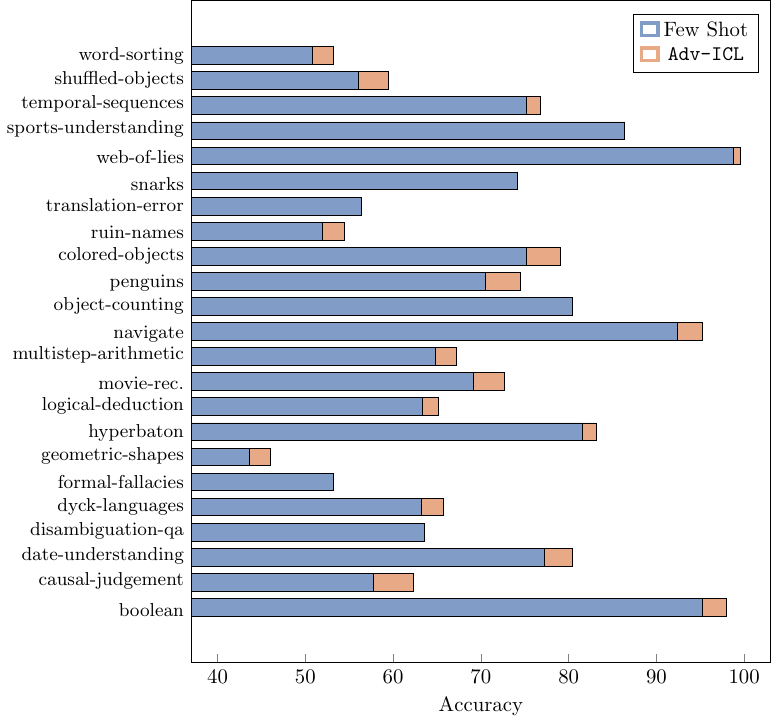}
    \caption{Full results on BBH using ChatGPT and 5-shot CoT prompting.}
    \label{fig:bbh_full}
\end{figure*}

\end{document}